%% file: main.tex
\documentclass{article}

\input{packages}

\input{macros}

\usepackage[top=1in,bottom=1in,left=2in,right=2in]{geometry}

\title{\vspace{-4em}Computing the Information Content\\of Trained Neural Networks}
\date{}
\author{\large Jeremy Bernstein\\ \normalsize \href{mailto:bernstein@caltech.edu}{bernstein@caltech.edu}
\and \large Yisong Yue\\ \normalsize \href{mailto:yyue@caltech.edu}{yyue@caltech.edu} }

\begin{document}
\maketitle
\input{section/1-abstract}
\tableofcontents
\thispagestyle{empty}

\newpage
\input{section/2-intro}

\newpage
\input{section/3-related}
\newpage
\input{section/3.5-info}

\newpage
\input{section/4-pac-bayes}
\newpage
\input{section/5-priors}
\newpage
\input{section/6-info}
\newpage
\input{section/7-expt}
\input{section/8-discuss}

\bibliography{refs}
\bibliographystyle{plainnat}

\newpage
\appendix
\input{section/99-app}

\end{document}

%% file: packages.tex
\usepackage{subfigure}

\usepackage{hyperref}

\usepackage[utf8]{inputenc} %
\usepackage[T1]{fontenc}    %
\usepackage[round]{natbib}
\usepackage{url}            %
\usepackage{booktabs,array} %
\usepackage{nicefrac}       %
\usepackage[scaled=.65]{beramono}
\usepackage[final]{microtype}      %
\usepackage[final]{graphicx}
\usepackage{framed}
\usepackage{enumitem}
\usepackage{amsfonts,amssymb,amsthm,amsmath}
\usepackage{bm}
\usepackage{algorithm,algorithmic}
\usepackage{thmtools,thm-restate}
\usepackage{resizegather}
\usepackage{tikz}
  {
    \theoremstyle{definition}
    \newtheorem{definition}{Definition}
      \theoremstyle{plain}

  }

\usepackage[subfigure]{tocloft}

\makeatletter
\newcommand\thefontsize[1]{{#1 The current font size is: \f@size pt\par}}
\makeatother

%% file: macros.tex
\usepackage{dsfont}
\usepackage{stmaryrd}

\newcommand{\econst}{\mathrm{e}}

\newcommand{\eps}{\varepsilon}

\renewcommand{\phi}{\varphi}

\newcommand{\half}{\frac{1}{2}}

\newcommand{\R}{\mathbb{R}}

\newcommand{\sign}{\operatorname{sign}}

\newcommand{\trace}{\operatorname{tr}}

\newcommand{\abs}[1]{\vert {#1} \vert}
\newcommand{\norm}[1]{\Vert {#1} \Vert}

\newcommand{\diff}{\mathrm{d}}
\newcommand{\idiff}{\,\diff}

\newcommand{\Expect}{\operatorname{\mathbb{E}}}

\newcommand{\Probe}{\mathbb{P}}

%% file: section/1-abstract.tex
\begin{abstract}
    \noindent How much information does a learning algorithm extract from the training data and store in a neural network's weights? Too much, and the network would overfit to the training data. Too little, and the network would not fit to anything at all. Na\"{i}vely, the amount of information the network stores should scale in proportion to the number of trainable weights. This raises the question: \textit{how can neural networks with vastly more weights than training data still generalise?} A simple resolution to this conundrum is that the number of weights is usually a bad proxy for the actual amount of information stored. For instance, typical weight vectors may be highly compressible. Then another question occurs: \textit{is it possible to compute the actual amount of information stored?} %
    This paper derives both a consistent estimator and a closed-form upper bound on the information content of infinitely wide neural networks. The derivation is based on an identification between neural information content and the negative log probability of a Gaussian orthant. This identification yields bounds that analytically control the generalisation behaviour of the entire solution space of infinitely wide networks. The bounds have a simple dependence on both the network architecture and the training data. Corroborating the findings of \citet{valle-perez2018deep}, who conducted a similar analysis using approximate Gaussian integration techniques, the bounds are found to be both non-vacuous and correlated with the empirical generalisation behaviour at finite width.

\end{abstract}

%% file: section/2-intro.tex
\section{Introduction}

As neural networks migrate from research labs to production systems, understanding their inner workings is becoming increasingly important. Yet there is still little consensus on some of the most basic questions. How much information does a network extract from the data during training? Why does a network with the capacity to memorise the training data typically not do so? What inductive bias underlies the network's ability to generalise? A firm understanding of these matters could help to explain why learning systems are so capable in some respects---and so unreliable in others.

When designing a statistical model to generalise beyond a set of training points, a traditional rule of thumb requires that:
\begin{equation}\label{eq:rot1}
    \#\,\mathrm{model\,parameters} < \#\, \mathrm{training\,data}.
\end{equation}
Empirically, and perhaps mysteriously, neural networks that violate this rule of thumb can still generalise. This matter was raised in \citet{radford}'s Ph.D.\ thesis, which ``challenge[d] the common
notion that one must limit the complexity of the model used when the amount of training
data is small'', and studied networks with \textit{infinitely more weights than data}. The general issue has also been broached from the learning theoretic angle. For instance, \citet{linial} ``establish[ed] the learnability of various concept classes with an \textit{infinite Vapnik-Chervonenkis dimension}''. In other words, the authors provided generalisation guarantees for models flexible enough to memorise any possible labelling of the training data.

\citeauthor{linial}'s work helped seed the PAC-Bayesian generalisation theory \citep{ShaweTaylor1996AFF,McAllester}. While this theory will be introduced rigorously in Section \ref{sec:pac}, for now some intuition will suffice. PAC-Bayes relates generalisability to the \textit{volume of parameter space excluded by the training data}. If fewer solutions are excluded, then less information needs to be extracted from the training data to find just \textit{one} solution, making memorisation of the training data less likely. This idea turns out to be equivalent to the statement that generalisation is guaranteed when \textit{the information content of a typical solution is sufficiently small}. Under PAC-Bayesian theory, the rule of thumb (\ref{eq:rot1}) may be refined to:
\begin{equation}\label{eq:rot2}
    \#\,\mathrm{bits\,needed\,to\,encode\,a\,solution} < \# \,\mathrm{training\,data}.
\end{equation}

A major question arises: how does one compute the information content of a trained neural network? This paper derives an analytical formula that takes as input a set of training data and an (infinitely wide) neural architecture, and returns the typical information content of the entire solution space. The formula is used to bound the test error averaged over all solutions. The appeal of this approach lies in the fact that it does not depend on the details of any specific learning algorithm or training procedure. By characterising the typical generalisation behaviour of the entire solution space, any learning algorithm that selects somewhat typical solutions will inherit the favourable generalisation properties of the space.

The paper is structured as follows: 
Section \ref{sec:related} discusses related work and distinguishes this paper's contributions.
Section \ref{sec:info} defines the relevant notion of information content. 
Section \ref{sec:pac} provides an introduction to PAC-Bayesian generalisation theory.
Section \ref{sec:symm} discusses a classic approach to bounding information content based on counting network symmetries.
Section \ref{sec:inf} introduces the neural network--Gaussian process correspondence, and characterises the information content of infinitely wide networks. Finally, Section \ref{sec:expt} experimentally tests the resulting generalisation bound.

%% file: section/3-related.tex
\section{Related Work}
\label{sec:related}

\textit{When we read a book or listen to a piece of music, how much of what we consume becomes a part of us?}

Quantifying the amount of information absorbed by a neural system through learning is an important question in artificial intelligence, statistics and the life sciences.

\paragraph{Information in a living system.} A folkloric estimate of the information content of the human brain is $\mathcal{O}(10^{14})$ bits based on counting the number of synapses. This estimate may be refined by combining it with an estimate of the information content of a single synapse \citep{nanoconnectomic}. The idea of estimating information content via counting synapses is directly analogous to parameter counting in statistics---see rule of thumb (\ref{eq:rot1}).

Still, cognitive scientists have criticised synapse counting as a means of estimating information content, since it is not based on a mechanistic understanding of how information is actually represented and stored in the brain. For instance, \citet{landauer} estimated the ``functional information content'' of human memory to be $\mathcal{O}(10^{9})$ bits---much lower than suggested by synapse counting. More recent work in the same vein was conducted by \citet{Mollica2019HumansSA} in the context of language acquisition.

To illustrate the fundamental importance of estimating neural information content, consider an application. Quantifying the amount of information gleaned through learning is central to understanding the relative importance of learned and innate behaviours in the animal kingdom---the \textit{nature versus nurture} debate. For instance, \citet{genomic} explored the ``genomic bottleneck''---the idea that the neural information that encodes innate behaviours must be compressed and stored in the genome.

\paragraph{Information in an artificial intelligence.} The connection between the information content of a statistical model and its predictive performance is front and centre in the \textit{minimum description length} theory of learning \citep{rissanen1986}. \citet{Hinton1993KeepingTN} explored this idea in the context of neural networks. However, \citet{rissanen1986} equates description length with parameter count, not addressing the \textit{actual} information content of a model. Likewise, the Akaike Information Criterion \citep{akaike} and the Bayesian Information Criterion \citep{bic} advocate for statistical model comparison based on parameter counting.

The PAC-Bayesian generalisation theory \citep{ShaweTaylor1996AFF, ShaweTaylor1997APA, McAllester}, on the other hand, relates predictive performance of a statistical model to its \textit{actual} information content. This is usually quantified in terms of the Kullback–Leibler divergence between a prior over models and the possible outcomes of the learning procedure. PAC-Bayes may be interpreted as a formalisation of Occam's Razor \citep{BLUMER1987377}---with a precise information theoretic notion of simplicity. Related ideas were explored by \citet{SCHMIDHUBER1997857}, but focusing on the Kolmogorov interpretation of complexity. \citet{DR17} investigated PAC-Bayes for neural networks. The authors derived non-vacuous generalisation bounds by computing a numerical estimate of the information content of solutions close to the weight initialisation.

The idea that neural networks are biased towards simple, low information content functions is often described as \textit{simplicity bias} in modern neural networks research. For instance, \citet{pmlr-v70-arpit17a} find that neural networks tend to first extract simple patterns from the training data, and argue that dataset dependent notions of neural network complexity are therefore needed. \citet{valle-perez2018deep}, \citet{NEURIPS2019_feab05aa} and \citet{Mingard2020IsSA} go further, suggesting that the entire space of neural networks is dominated by simple functions. Relatedly, \citet{Ulyanov2018DeepIP} explore the inductive bias of neural networks, finding that randomly initialised convolutional networks serve as a good prior for image-based inverse problems.

\paragraph{Infinite width limits} This paper builds on a body of work exploring neural networks as the number of hidden units in a layer is taken to infinity. In particular, the paper focuses on \citet{radford}'s formulation: under random sampling of the weights of an infinitely wide neural network, the distribution of network outputs for a fixed set of inputs follows a multivariate Normal distribution---with a covariance matrix determined by the network architecture and the inputs. Modern research refers to this formulation as the \textit{neural network--Gaussian process} (NNGP) correspondence \citep{lee2018deep}. 

Importantly, this paper does \textit{not} use the \textit{neural tangent kernel} (NTK) formulation due to \citet{NEURIPS2018_5a4be1fa}, which attempts to build an infinite width characterisation of training dynamics under gradient descent. \citet{NEURIPS2019_0d1a9651} interpret the NTK theory as saying that below a certain learning rate, infinitely wide neural networks never escape their linearisation about their initial parameter values. But whether this NTK limit adequately captures the full power of deep neural networks is unclear \citep{NEURIPS2019_ae614c55}.

By focusing on randomly sampled networks under the NNGP formulation and thereby accessing properties of the entire neural network function space, this paper avoids the need for characterisations of training dynamics such as NTK. In particular, by characterising the typical properties of the entire space of solutions, the results in this paper should extend to any learning method that selects somewhat typical solutions from the solution space.

\paragraph{This paper's contributions} The primary contribution of this paper is to derive a consistent estimator and a closed form upper bound on the information content of infinitely wide networks (Theorem \ref{thm:info}) through a characterisation of Gaussian orthant probabilities (Lemma \ref{thm:orthant}). This extends an analysis made by \citet{valle-perez2018deep}, where the authors used approximate Gaussian integration techniques and did not derive closed form expressions for the information content.

The second contribution is a generalisation bound for infinitely wide networks (Theorem \ref{thm:gen}) with an analytical dependence on network architecture and training data. At its heart, this is a PAC-Bayesian generalisation bound for Gaussian process classification combined with the NNGP characterisation of infinite width networks. By making the identification between the exact Gaussian process posterior and a Gaussian orthant, this paper avoids the need for Gaussian process approximations considered by \citet{seeger}. Furthermore, the generalisation bound is complementary to the PAC-Bayesian analysis of \citet{DR17}. Whereas this paper derives analytical results on the \textit{global} measure of solutions across the entire function space of infinite width networks, \citet{DR17} numerically compute the \textit{local} measure of solutions close to the weight initialisation of finite width networks. Both approaches lead to non-vacuous generalisation bounds, although this paper's approach is perhaps more analytically interpretable.

The final contribution is an experimental test of the derived generalisation bound for wide binary classifiers. Corroborating the findings of \citet{valle-perez2018deep}, this paper's Theorem \ref{thm:gen} is found to be both non-vacuous and correlated with the empirical generalisation behaviour at finite width.

%% file: section/3.5-info.tex
\input{figures/version}

\section{Information: Shannon, Kolmogorov \& Bayes}
\label{sec:info}

This section establishes the rigorous notion of information that shall be considered throughout the paper. The information is defined with respect to the entire solution space of a classifier---as depicted in Figure \ref{fig:version-space}.
\begin{definition}
For a classifier with weight space $\Omega$, the \textit{verson space} $V_S\subset\Omega$ denotes the subset of weight settings that fit a training set $S$ without error.
\end{definition}
\begin{definition}\label{def:info}
Fix a probability measure $\Probe$ on $\Omega$. Then the \textit{typical information content} $\mathcal{I}$ of the version space $V_S\subset\Omega$ is given by
$\mathcal{I}[V_S] := \ln \frac{1}{\Probe[V_S]}$.
\end{definition}
To gain intuition, consider a finite weight space $\Omega$ of $N$ elements, let $\Probe$ be uniform over $\Omega$, and let $V_S$ contain $K$ solutions. Then the typical information content $\mathcal{I}[V_S]=\ln (N/K)=\ln N - \ln K$. In words, it is the ``number of parameters'' $\ln N$ needed to index an element of $\Omega$, less the ``entropy of solutions'' $\ln K$. Thus parameter counting overestimates the information content for models with a redundancy in the solution space.
The typical information content entertains further interpretations, as follows:

First, the \textbf{Shannon information} of an event with probability $p$ is $\log 1/p$. Therefore, the typical information content $\mathcal{I}[V_S]$ is the Shannon information content of the version space under measure $\Probe$. This quantifies the degree of surprise one experiences upon randomly sampling a classifier according to measure $\Probe$ and obtaining a member of the version space $V_S$.

Second, the \textbf{Kolmogorov complexity} of some data is the length of the shortest computer program that produces that data---so it is a theoretical measure of compressibility. Assuming the existence of a good pseudorandom number generator, $\mathcal{I}[V_S]$ gives an upper bound on the Kolmogorov complexity of typical members of the version space $V_S$. To see this, consider drawing classifiers from $\Probe$ until a member of the version space is discovered by chance. Such a solution may be described efficiently by the random seed and the total number of draws. But the expected number of draws required is $1/\Probe[V_S]$. Then Markov's inequality implies that with greater than $99.9\%$ probability, the number of bits $b$ needed to store the required number of draws satisfies $b \leq \log_2 (1/\Probe[V_S]) + 10$. In words, typical solutions may be described in a number of bits close to the typical information content.

And finally, the Bayesian interpretation. One may think of $\Probe$ as defining a Bayesian prior on the weight space $\Omega$. Choosing the likelihood function to be the indicator $\bm{1}_{V_S}$, then the \textbf{Bayesian evidence} for the class of networks given the training data is given by $\int_\Omega \bm{1}_{V_S} \idiff{\Probe}=\Probe[V_S]=\exp(-\mathcal{I}[V_S])$.

%% file: figures/version.tex
\begin{figure}
\begin{center}
    \begin{tikzpicture}[fill=gray!30]
    \fill  plot[smooth, tension=.9, scale=.4, shift={(-2,-2)}] coordinates {(-3.5,0.5) (-3,2.5) (-1,3.5) (1.5,3) (4,3.5) (5,2.5) (5,0.5) (2.5,-2) (0,-0.5) (-3,-2) (-3.5,0.5)};
    \draw  plot[smooth, tension=.9, scale=.4, shift={(-2,-2)}, thick] coordinates {(-3.5,0.5) (-3,2.5) (-1,3.5) (1.5,3) (4,3.5) (5,2.5) (5,0.5) (2.5,-2) (0,-0.5) (-3,-2) (-3.5,0.5)}
    (-3,-2) rectangle (3,2);
    \node at (1.6,1.7) {\textbf{weight space $\Omega$}};
    \node at (1.65,1.2) {volume $\Probe[\Omega]=1$};
    \node at (-.4,-.05) {\textbf{version space $V_S$}};
    \node at (-.4,-.55) {volume $\Probe[V_S]$};
    \end{tikzpicture}
    \caption{A classifier's version space $V_S$ is the subset of weight space $\Omega$ that fits a training set $S$ without error. To measure the relative volume of the version space, one can equip $\Omega$ with a probability measure $\Probe$.}%
    \label{fig:version-space}
\end{center}
\end{figure}
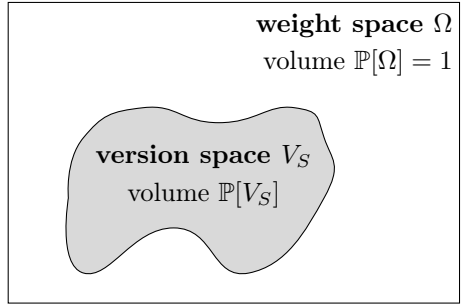

%% file: section/4-pac-bayes.tex
\section{PAC-Bayesian Generalisation Theory}
\label{sec:pac}

The PAC-Bayesian method connects information content with generalisation.

Consider training a machine learning model to fit $n$ training points with binary labels. Suppose that (i) the trained model exactly fits the training labels, and (ii) the model extracts $b$ bits of information from the training set during learning. If more bits are extracted than there are binary training labels ($b\geq n$), then it is possible that the model memorised the training labels. But if the reverse is true ($b<n$), then the model cannot have memorised the training labels and must have captured some structure in the data. In the extreme that $b\ll n$, the model has found a simple rule that fits the training data and would be expected to generalise well.

This informal argument is made concrete by the PAC-Bayesian theorem. When the version space is non-empty, the training set is said to be \textit{realisable} by the classifier, and the theorem becomes very simple. When reading the theorem in the context of deep learning, it may help to think of the measure $\Probe$ as a weight initialiser, such as \textit{Xavier init} \citep{pmlr-v9-glorot10a}.

\begin{restatable}[Realisable PAC-Bayes]{theorem}{pacbayes}
\label{thm:pacbayes}
First, fix a probability measure $\Probe$ over the weight space $\Omega$ of a classifier. Let $S$ denote a training set of $n$ iid datapoints, $V_S\subset\Omega$ its version space, and $\mathcal{I}[V_S]$ the typical information content. Consider the population error rate $0\leq \eps(w)\leq 1$ of weight setting $w\in\Omega$, and its average over the version space $\eps(V_S) := \Expect_{w\sim\Probe}[\eps(w)|w\in V_S]$. Then, for a proportion $1-\delta$ of random draws of the training set $S$,
\begin{equation}
     \eps(V_S) \leq \ln \frac{1}{1-\eps(V_S)}\leq \frac{\mathcal{I}[V_S] + \ln \frac{2n}{\delta}}{n-1}.
\end{equation}
\end{restatable}
In words, the average test error over the version space (the entire space of solutions) is low when the number of datapoints $n$ greatly exceeds the typical information content $\mathcal{I}[V_S]$ of the version space. This version of the PAC-Bayesian theorem was proved by \citet{valle-perez2018deep}. The proof is included in Appendix \ref{app:proofs} for completeness.

PAC refers to the fact that the bound only holds for most training sets, so the learning procedure is only \textit{probably approximately correct}. Since the failure probability $\delta$ appears inside a logarithm, the PAC aspect becomes unimportant when the number of data points $n$ is large.

The significance of Theorem \ref{thm:pacbayes} is perhaps under-appreciated in the machine learning community. Observe that if the average test error over the version space $\eps(V_S)$ is low, then \textit{most} members of the version space must have low test error (either intuitively, or by Markov's inequality). This view contrasts to papers that emphasise deriving PAC-Bayesian generalisation guarantees only for \textit{ensembles of networks} \citep{DR17}. To reiterate, the spirit of PAC-Bayes is as follows:
\begin{quote}
    \textit{Provided the typical information content $\mathcal{I}[V_S]$ of the version space is small compared to the number of datapoints $n$, then poorly generalising solutions are rare.}
\end{quote}
Or rephrased in Bayesian terminology: %
\begin{quote}
    \textit{Provided the Bayesian evidence $\Probe[V_S]$ for the model given the data is large compared to the reciprocal exponential of the number of datapoints $\econst^{-n}$, then a typical solution will generalise well.}
\end{quote}
Finally, PAC-Bayes formalises \textit{Occam's razor} \citep{BLUMER1987377}: 
\begin{quote}
    \textit{A model is ``simple'' if the Bayesian evidence $\Probe[V_S]\gg\econst^{-n}$, regardless of the number of parameters.}
\end{quote}

%% file: section/5-priors.tex
\section{Information via Symmetry Counting}\label{sec:symm}

\input{figures/nn}

This section discusses a classic approach to bounding the information content of an $L$-layer multilayer perceptron (MLP). Define the network recursively:
\begin{align}
    z^{(1)}_i(x) &= \frac{1}{\sqrt{d_0}}\sum_{j=1}^{d_0} W^{(1)}_{ij} x_j, \label{eq:first-layer}\\
    z^{(l)}_i(x) &= \frac{1}{\sqrt{d_{l-1}}}\sum_{j=1}^{d_{l-1}} W^{(l)}_{ij}\phi \left(z^{(l-1)}_j(x)\right), \label{eq:layer}
\end{align}
where $x\in\R^{d_0}$ denotes an input, $z^{(l)}(x)\in\R^{d_l}$ denotes the pre-activations at the $l$th layer, $W^{(l)}$ denotes the weight matrix at the $l$th layer, and $\phi$ denotes the nonlinearity. A three layer network is visualised in Figure \ref{fig:nn}.

Since a fully connected network is symmetric under permutation of its hidden units, then assuming its version space $V_S$ is non-empty, $V_S$ must contain at least $\prod_{l=1}^{L-1}d_l!$ functions by symmetry. \citet{denker} used this approach to bound the information content by comparing the number of symmetries to the total number of weight vectors:
\begin{equation}\label{eq:symm}
    \Probe[V_S] \geq \frac{\#\,\mathrm{symmetries}}{\#\,\mathrm{weight\,vectors}} = \frac{\prod_{l=1}^{L-1}d_l!}{2^{\sum_{l=1}^{L}d_l d_{l-1}w}}.
\end{equation}
The denominator is due to the network containing a number $\sum_{l=1}^{L}d_ld_{l-1}$ of $w$-bit weights. Unfortunately, this approach seems to become vacuous at large width. To see this, observe that Equation \ref{eq:symm} leads to the following:
\begin{align}
    \mathcal{I}[V_S] &\leq \ln(2) \sum_{l=1}^{L}d_l d_{l-1}w - \sum_{l=1}^{L-1} \ln(d_l!) \label{eq:info-symm}\\
    &\approx \ln(2)d_Ld_{L-1}w + \sum_{l=1}^{L-1}\left[ \ln(2)d_ld_{l-1}w - d_l\ln(d_l) + d_l\right].
\end{align}
where the second line follows from Sterling's approximation of the factorial. But at large width $d:=d_1=d_2=...=d_{L-1}$, this bound goes like $\mathcal{O}(wLd^2)$. This is proportional to the number of weights $Ld^2$, and in the limit that the width $d\rightarrow \infty$, the bound on information content becomes infinite.

So simple permutation symmetry does not appear to account for the generalisability of vastly over-parameterised networks. It should be said that this does not rule out the existence of other kinds of tricky-to-count symmetry. \citet[Appendix B]{DR17} provide further discussion.

%% file: figures/nn.tex
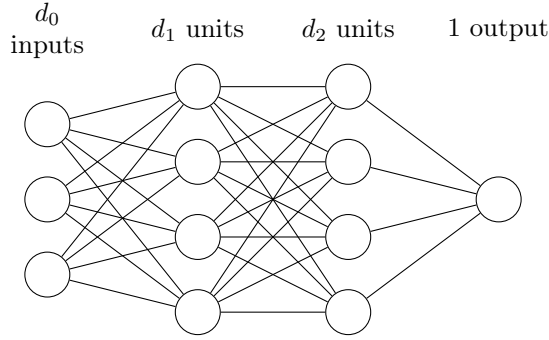
\begin{figure}
\begin{center}
\def\layersep{2cm}
\begin{tikzpicture}[-,draw=black, node distance=\layersep]
    \tikzstyle{neuron}=[circle,draw=black,minimum size=17pt,inner sep=0pt]
    \tikzstyle{annot} = [text width=4em, text centered]

    \foreach \name / \y in {1,...,3}
        \node[neuron] (I-\name) at (0,-\y) {};

    \foreach \name / \y in {1,...,4}
        \path[yshift=0.5cm]
            node[neuron] (H-\name) at (\layersep,-\y cm) {};
            
    \foreach \name / \y in {1,...,4}
        \path[yshift=0.5cm]
            node[neuron] (H1-\name) at (\layersep*2,-\y cm) {};

    \node[neuron] (O) at (\layersep*3,-2 cm) {};

    \foreach \source in {1,...,3}
        \foreach \dest in {1,...,4}
            \path (I-\source) edge (H-\dest);
            
    \foreach \source in {1,...,4}
        \foreach \dest in {1,...,4}
            \path (H-\source) edge (H1-\dest);

    \foreach \source in {1,...,4}
        \path (H1-\source) edge (O);

    \node[annot,above of=H-1, node distance=0.75cm] (hl) {$d_1$ units};
    \node[annot,left of=hl] {$d_0$ inputs};
    \node[annot,right of=hl] (h2) {$d_2$ units};
    \node[annot,right of=h2] {$1$ output};
\end{tikzpicture}
\end{center}

\caption{A multilayer perceptron (MLP) with $L=3$ layers.}
\label{fig:nn}
\end{figure}

%% file: section/6-info.tex
\section{Information via Infinite Width}\label{sec:inf}

The symmetry counting approach failed in the previous section because, as the network grows wider, the number of weight vectors grows far faster than the number of permutation symmetries. But one might have hoped for life to become \textit{easier} as the network grows wider. If each individual weight became less important as the width tends to infinity, then the function space might actually \textit{converge}. This idea is made rigorous by the neural network--Gaussian process correspondence \citep[NNGP]{radford}. The implication for this paper is that the information content of solutions will also converge.

\subsection{Neural Networks as Gaussian Processes}

Consider a training data set $S$ containing $n$ datapoints $x_1,...,x_n\in\R^{d_0}$ each with a binary class label $c_1,...,c_n\in\{\pm1\}$, and consider an infinitely wide network $f:\R^{d_0}\to\R$ sampled at random from the entire space of such networks $\Omega$ equipped with measure $\Probe$. Then under mild conditions (e.g. Theorem \ref{thm:relu}), the distribution of outputs over the training data is given by:
\begin{align}
    f(x_1), ... f(x_n) \sim \mathcal{N} (0,\Sigma),\label{eq:jointnormal}\\
    \Sigma_{ij} := \Expect_{f\sim\Probe} [f(x_i)f(x_j)].\label{eq:sigma}
\end{align}

This means that the entire function space of the network is characterised by the covariance structure $\Sigma(x,x^\prime) = \Expect_{f\sim\Probe} [f(x)f(x^\prime)]$. Therefore, in this infinite width limit, one expects the typical information content $\mathcal{I}[V_S]$ to be expressible purely as a function of $\Sigma$ and the training data set $S$.

The covariance structure $\Sigma(x,x^\prime)$ will generally depend on the the overall network topology. An example will help to clarify this. The following theorem, based on the results of \citet{choandsaul} and \citet{lee2018deep}, evaluates the covariance $\Sigma$ for depth-$L$ relu MLPs.

\begin{restatable}
[NNGP for relu networks]{theorem}{relu}\label{thm:relu} Consider an $L$-layer MLP defined recursively via Equations \ref{eq:first-layer} and \ref{eq:layer}, with input dimension $d_0$, hidden dimensions $d_1,...,d_{L-1}$ and output dimension $d_L=1$. Set the nonlinearity to a scaled relu: $$\phi(z) := \sqrt{2}\cdot\max(0,z).$$
Suppose that the weight matrices $W^{(1)}, ..., W^{(L)}$ have entries drawn iid $\mathcal{N}(0,1)$, and consider any collection of $k$ inputs $x_1, ..., x_k$ each with Euclidean norm $\sqrt{d_0}$. 

If $d_1,...,d_{L-1}\to \infty$, the distribution of outputs $z^{(L)}(x_1), ..., z^{(L)}(x_k) \in \R$ induced by random sampling of the weights is jointly Normal with:
\begin{align*}
\Expect \left[ z^{(L)}(x) \right] &= 0;\\
\Expect \left[ z^{(L)}(x)^2 \right] &= 1;\\ 
\Expect \left[ z^{(L)}(x) z^{(L)}(x^\prime)\right] &= \underbrace{h \circ ... \circ h}_{L-1\text{  times}}\left(\frac{x^Tx^\prime}{d_0}\right);
\end{align*}
where $h(t):= \tfrac{1}{\pi}\left[\sqrt{1-t^2} + t\cdot(\pi- \arccos t)\right]$.
\end{restatable}

The proof is given in Appendix \ref{app:proofs}. This covariance structure is referred to as the \textit{compositional arccosine kernel} \citep{choandsaul}. Notice that according to Theorem \ref{thm:relu}, a single input $x$ with Euclidean norm $\sqrt{d_0}$ induces an output $z^{(L)}(x) \sim \mathcal{N}(0,1)$. This output standardisation results from the careful scalings adopted for the input, weight matrices and nonlinearity.

\subsection{Gaussian Orthant Probabilities}

Does the NNGP correspondence facilitate computation of the typical information content of infinitely wide neural networks?

Observe that the version space $V_S$ may be expressed as:
\begin{align}
    V_S&=\{f\in\Omega : \sign f(x_1)=c_1,...,\sign f(x_n)=c_n\}. \label{eq:vs}
\end{align}
Then in the infinite width limit, by Equation \ref{eq:jointnormal}, the measure of the version space is given by:
\begin{align}\label{eq:ivs}
    \Probe[V_S] &= \Probe_{z\sim\mathcal{N}(0,\Sigma)}[\sign z_1 = c_1,...,\sign z_n = c_n].
\end{align}
Equation \ref{eq:ivs} is just the probability of the Gaussian orthant picked out by the binary classes $c_1,...,c_n$. Finally, by Definition \ref{def:info}, the information content in the infinite width limit is given by:
\begin{align}\label{eq:info-iwl}
    \mathcal{I}[V_S] &= \ln\left(1/ \Probe_{z\sim\mathcal{N}(0,\Sigma)}[\sign z_1 = c_1,...,\sign z_n = c_n]\right).
\end{align}

This identification between information content and Gaussian orthant probabilities is only useful provided there is a means to compute, estimate or bound Gaussian orthant probabilities. A literature exists on the topic of high-dimensional Gaussian integration, and  \citet{Genz2009ComputationOM} provide a review. Typically, exact computations are tractable only for special cases, and otherwise approximations are used. 

To make progress, this paper has derived the following estimator and bound on Gaussian orthant probabilities:
\begin{restatable}[Gaussian orthant probability]{lemma}{orthant}\label{thm:orthant}
For a covariance matrix $\Sigma\in\R^{n\times n}$, and a binary vector $c\in\{\pm1\}^n$, let $p$ denote the corresponding Gaussian orthant probability:
\begin{equation}
    \label{eq:orthant}
    p := \Probe_{z\sim\mathcal{N}(0,\Sigma)}[\sign(z)=c].
\end{equation}
Letting $\mathbb{I}$ denote the $n\times n$ identity matrix, $\odot$ the elementwise product and $\abs{\cdot}$ the elementwise absolute value, then $p$ may be equivalently expressed as:
\begin{equation}
    \label{eq:estimator}
    p = \frac{1}{2^n} \Expect_{u\sim\mathcal{N}(0,\mathbb{I})}\left[\econst^{- \half (c\odot\abs{u})^T\left(\sqrt[n]{\det{\Sigma}}\Sigma^{-1} - \mathbb{I}\right)(c\odot\abs{u})}\right],
\end{equation}
and $p$ may be bounded as follows:
\begin{equation}
    \label{eq:bound}
    p \geq \frac{1}{2^n}\econst^{\frac{n}{2}-\sqrt[n]{\det \Sigma}\left[ (\frac{1}{2}-\frac{1}{\pi})\trace(\Sigma^{-1}) + \frac{1}{\pi} c^T\Sigma^{-1}c\right]}.
\end{equation}
\end{restatable}

To gain intuition about the lemma, notice that $1/2^n$ is the orthant probability for an isotropic Gaussian. Depending on the degree of \textit{anisotropy} $\sqrt[n]{\det{\Sigma}}\Sigma^{-1} - \mathbb{I}$ inherent in the covariance matrix $\Sigma$, Equation \ref{eq:estimator} captures how the orthant probability may either be exponentially amplified or exponentially suppressed compared to $1/2^n$.

In contrast to more general purpose Gaussian integration methods such as the \textit{GHK simulator} \citep[Chapter 4.2]{Genz2009ComputationOM}, Lemma \ref{thm:orthant} directly exploits the symmetric structure of Gaussian orthants. The proof of the lemma is given in Appendix \ref{app:proofs}, and constitutes one of this paper's main technical contributions.

\newpage
\subsection{Neural Information Content}

With the correspondence to Gaussian orthant probabilities in place, it is now a simple matter to derive the information content of infinitely wide neural networks and their corresponding generalisation theory. First, it will help to define the following two \textit{kernel complexity measures}:
\begin{align}
    \mathcal{C}_0(\Sigma, c) &:= \ln\frac{2^n}{ \Expect_{u\sim\mathcal{N}(0,\mathbb{I})}\left[\econst^{- \half (c\odot\abs{u})^T\left(\sqrt[n]{\det{\Sigma}}\Sigma^{-1} - \mathbb{I}\right)(c\odot\abs{u})}\right]}\geq 0,\label{eq:c0}\\
    \mathcal{C}_1(\Sigma, c) &:= \sqrt[n]{\det\Sigma}\left[\left(\half - \frac{1}{\pi}\right)\trace(\Sigma^{-1}) + \frac{1}{\pi}c^T\Sigma^{-1}c\right] \geq 0,\label{eq:c1}
\end{align}
where $\mathbb{I}$ denotes the $n\times n$ identity matrix, $\odot$ denotes the elementwise product, and $\abs{\cdot}$ denotes the elementwise absolute value. Then the following result is a basic consequence of Equation \ref{eq:info-iwl} and Lemma \ref{thm:orthant}.

\begin{restatable}[Neural information content]{theorem}{info}\label{thm:info}Consider a space of infinitely wide networks $\Omega$ equipped with a probability measure $\Probe$ that satisfies the requirements of the NNGP correspondence. See Theorem \ref{thm:relu} for an example. For a training set $S$ of $n$ data points $x_1,...,x_n$ with binary labels $c_1,...,c_n\in\{\pm1\}$, define covariance matrix $\Sigma\in\R^{n\times n}$ as in Equation \ref{eq:sigma}.

Then the typical information content of the version space is:
\begin{align}
    \mathcal{I}[V_S] = \mathcal{C}_0(\Sigma, c) \leq \frac{n}{5} + \mathcal{C}_1(\Sigma, c).
\end{align}
\end{restatable}

To gain some intuition about Theorem \ref{thm:info}, first consider the case that the neural architecture induces no correlation between any pair of distinct data points, such that $\Sigma = \mathbb{I}$. In this case, the information content $\mathcal{I}[V_S] = \mathcal{C}_0(\mathbb{I},c) = n\ln 2$, which scales with the number of data points $n$.

Next, suppose that the neural architecture induces strong intra-class correlations and strong inter-class anti-correlations, such that $\Sigma_{ij} = c_ic_j$. While Theorem \ref{thm:info} is not applicable since $\Sigma$ is singular, the version space probability is seen directly to be $\Probe[V_S] = 1/2$, since half of all draws from $\mathcal{N}(0,\Sigma)$ will achieve the correct class labelling. Therefore the information content $\mathcal{I}[V_S] = \ln 2$, which is independent of the number of data points.

The next result is a combination of Theorems \ref{thm:pacbayes} and \ref{thm:info}.

\begin{restatable}[Generalisation bound]{theorem}{gen}\label{thm:gen} In the same setting as Theorem \ref{thm:info}, let $0\leq \eps(f)\leq 1$ denote the test error of network $f\sim\Probe$, and let $\eps(V_S)$ denote the average test error over the version space: $\eps(V_S) := \Expect_{f\sim\Probe}[\eps(f)|f\in V_S]$. Then, for a proportion $1-\delta$ of randomly drawn iid training sets $S$,

\begin{align}
     \eps(V_S) \leq \ln \frac{1}{1-\eps(V_S)}&\leq \frac{ \mathcal{C}_0(\Sigma, c) + \ln \frac{2n}{\delta}}{n-1}\leq \frac{\frac{n}{5}+\mathcal{C}_1(\Sigma, c) + \ln \frac{2n}{\delta}}{n-1}.
\end{align}
\end{restatable}

Notice that for large $n$, the result simplifies to: $\eps(V_S)\lessapprox \frac{\mathcal{C}_0(\Sigma, c)}{n} \leq \frac{1}{5}+\frac{\mathcal{C}_1(\Sigma, c)}{n}$. In words, the typical test error depends on a tradeoff between the kernel complexity measures $\mathcal{C}_0,\mathcal{C}_1$ and the number of datapoints $n$.

Again, for the case of no correlations induced between inputs ($\Sigma = \mathbb{I}$), the bound on the test error is $\eps(V_S) \lessapprox \ln2$, which is worse than chance. This corresponds to \textit{pure memorisation} of the training labels. But for strong intra-class correlations and strong inter-class anti-correlations ($\Sigma_{ij} = c_ic_j$), the test error $\eps(V_S) \lessapprox \frac{\ln2}{n}$ (by Theorem \ref{thm:pacbayes}), which is much better than chance. This corresponds to \textit{pure generalisation} from the training labels.

Extrapolating from these observations would suggest that the role of a good neural architecture is to impose a prior on functions with strong intra-class correlations, and---if possible---strong inter-class anti-correlations.

%% file: section/7-expt.tex
\section{Experiments}
\label{sec:expt}

Theorem \ref{thm:gen} was tested against the generalisation performance of finite width networks trained by gradient descent, and across datasets of varying complexity. The Pytorch code is available here: {\fontsize{11}{12}\selectfont \href{https://github.com/jxbz/entropix}{\texttt{github.com/jxbz/entropix}}}.

\subsection{Datasets}\label{sec:data}
Three modified versions of the MNIST handwritten digit dataset \citep{lecun2010mnist} were used in the experiments:

\input{figures/datasets}

The anticipated difficulty of learning each dataset is listed in the final column. For \textit{random labels} there is no meaningful relationship between image and label, so \textit{memorisation} is the only option and \textit{learning} is impossible.

\subsection{Network Architecture and Training}

MLPs were trained with $L$ layers and $W$ hidden units per layer. Specifically, each $28\mathrm{px}\times28\mathrm{px}$ input image $x_i$ was flattened to lie in $\R^{784}$, and normalised to satisfy $\|x_i\|_2=\sqrt{784}$. The networks consisted of an input layer in $\R^{784\times W}$, $(L-2)$ layers in $\R^{W\times W}$, and an output layer in $\R^{W\times1}.$ The nonlinearity $\phi$ was set to $\phi(z):=\sqrt{2}\cdot\max(0,z)$. This matches the network defined in Equations \ref{eq:first-layer} and \ref{eq:layer} (with the $1/\sqrt{d_l}$ pre-factors absorbed into the weights), thus the limiting kernel is the \textit{compositional arccosine kernel} of Theorem \ref{thm:relu}.

The networks were trained for 20 epochs using the Nero optimiser \citep{nero2021} with an initial learning rate of 0.01 decayed by a factor of 0.9 every epoch, and a mini-batch size of 50 data points. The training error (measured as an average over the final epoch) was 0\% in almost all experiments. It was never larger than 2\%---even on \textit{random labels}.

\subsection{Numerical Stability}
When computing the quantities in Theorem \ref{thm:gen}, two techniques were used to improve stability. First, the Cholesky decomposition was used to obtain the covariance inverse and determinant, according to the formulae:
\begin{equation*}
    \textstyle\mathtt{\Sigma = LL^T; \quad
    \Sigma^{-1} = (L^{-1})^TL^{-1};\quad
    \sqrt[\mathtt{n}]{\mathtt{det}\,\Sigma} = \prod_{i=1}^n L_{ii}^{2/n}}.
\end{equation*}
It was verified that for all computations in the paper, $\mathtt{\norm{\Sigma \Sigma^{-1} - I}_\infty < 10^{-3}}$. And second, when the complexity $\mathcal{C}_0(\Sigma, c)$ was estimated via $10,000$ Monte-Carlo samples, the log-sum-exp decomposition was used to avoid overflow:
\begin{equation*}\textstyle
    \mathtt{\mathtt{ln} \sum_i e^{a_i} = a_* + \mathtt{ln} \sum_i e^{a_i-a_*}}, \text{ where } \mathtt{a_* := \mathtt{max}_i\,a_i}.
\end{equation*}

For all results, the mean and range are reported over three random seeds. The random seed not only affects the Monte-Carlo estimation of $\mathcal{C}_0(\Sigma,c)$, but also other sources of variability such as the iid draw of the training examples, the network initialisation and the data ordering. As can be seen in Figure \ref{fig:expt}, the variability was quite low in all experiments. For one thing, this suggests good convergence of the Monte-Carlo estimates.

All PAC-Bayes bounds were computed with failure probability $\delta=0.01$.

\subsection{Results}

First, the $\mathcal{C}_0$ (Equation \ref{eq:c0}) estimator and $\mathcal{C}_1$ (Equation \ref{eq:c1}) bound of Theorem \ref{thm:gen} were compared against the empirical test performance of a depth $L=7$ MLP trained on the \textit{decimal digits} dataset. The results are shown in Figure \ref{fig:expt} (left). As can be seen, the theoretical predictions appear slack compared to the empirical performance, but the bounds are still non-vacuous and display the correct trend with increasing number of training examples.

Next, the $\mathcal{C}_0$ estimator of Theorem \ref{thm:gen} was compared across the \textit{binary digits}, \textit{decimal digits} and \textit{random labels} datasets. The network architecture was again set to a depth $L=7$ MLP. The results are shown in Figure \ref{fig:expt} (middle). Notice that the estimator follows the trend of ``hardness'' across datasets listed in the table in Section \ref{sec:data}. Also, the estimator was vacuous (that is, worse than chance) for the \textit{random labels} dataset. This is desirable, since the \textit{random labels} dataset is impossible to generalise from.

Finally, the effect of varying network depth was investigated on the \textit{decimal digits} dataset. Two depths were compared: $L=2$ and $L=7$. The results are shown in Figure \ref{fig:expt} (right). While the estimators are fairly slack compared to the finite width results, the general trends with respect to both depth and number of training examples do match the finite width results.

\input{figures/expt}

%% file: figures/datasets.tex
\begin{table}[H]
    \centering
    \resizebox{\linewidth}{!}{
    \begin{tabular}{ccccc}
        \toprule
        \textbf{Name of Dataset} & \textbf{Input Images} & \textbf{Binary Labels} & \textbf{Hardness}\\ 
        \midrule
        binary digits & $\{0, 1\}$ & even $\mapsto-1$, odd $\mapsto+1$ & easy\\
        decimal digits & $\{0, 1,...,9\}$ & even $\mapsto-1$, odd $\mapsto+1$ & harder\\
        random labels & $\{0, 1,...,9\}$ & fair coin flip for each image & impossible\\
        \bottomrule
    \end{tabular}
    }
\end{table}

%% file: figures/expt.tex
\begin{figure}
    \centering
    \includegraphics[width=0.33\textwidth]{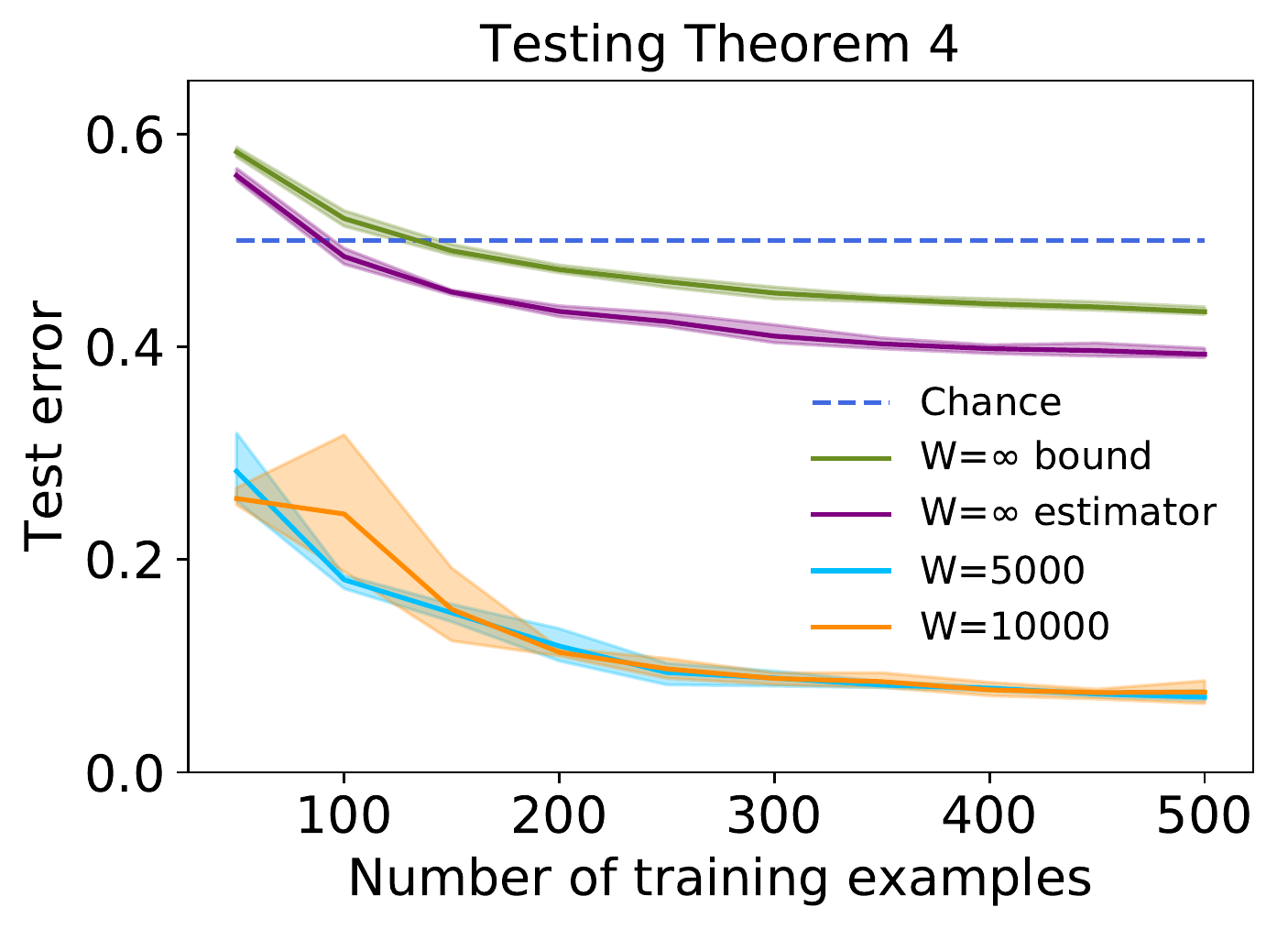}\includegraphics[width=0.33\textwidth]{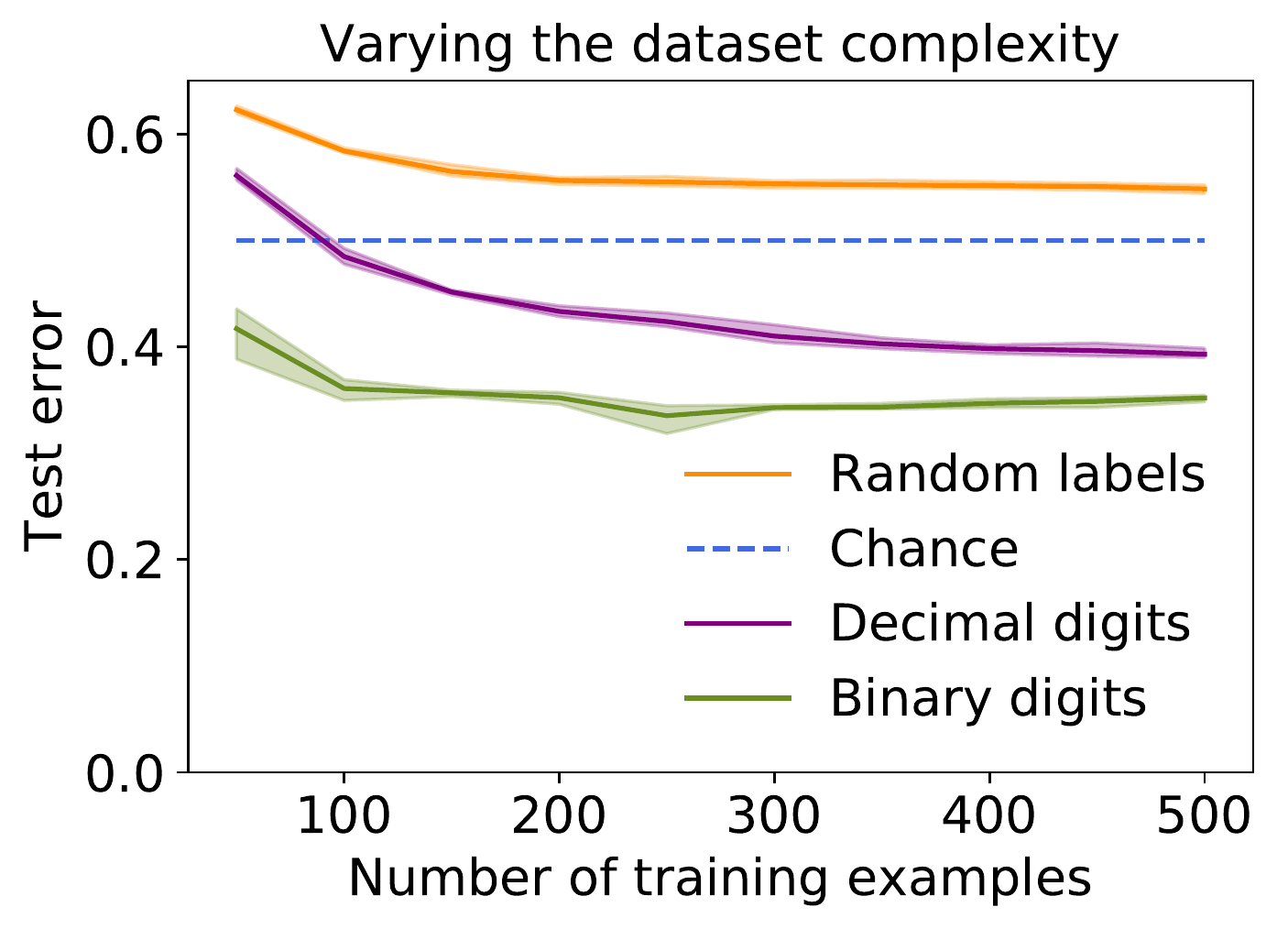}\includegraphics[width=0.33\textwidth]{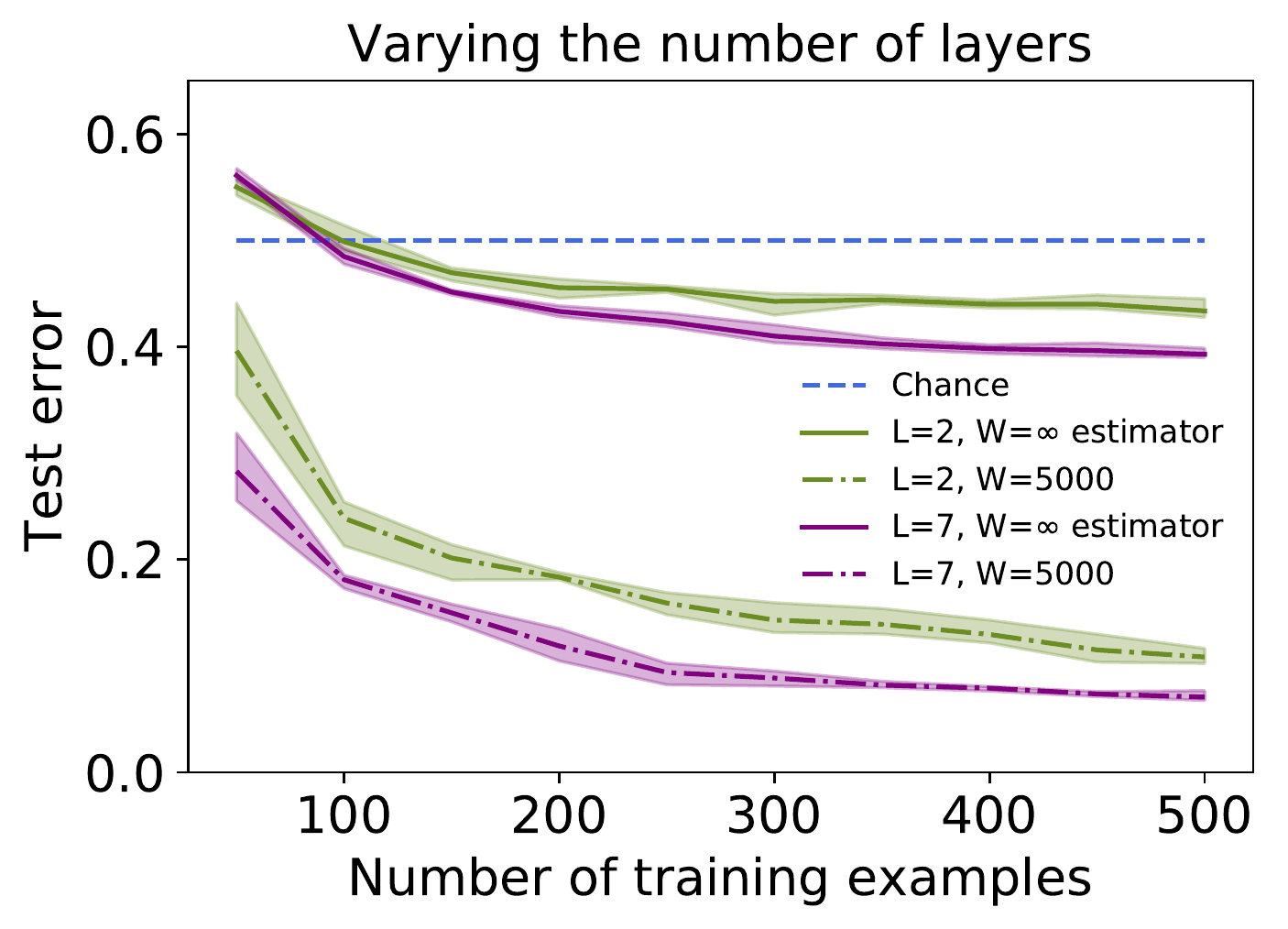}
    \caption{Left: Comparing the Theorem \ref{thm:gen} estimator and bound to the empirical test performance of depth $L=7$ MLPs on the \textit{decimal digits} dataset. Middle: Comparing the Theorem \ref{thm:gen} estimator across datasets of varying complexity, for depth $L=7$ MLPs. Right: Comparing the Theorem \ref{thm:gen} estimator across networks of varying depth on the \textit{decimal digits} dataset. Estimators were computed based on 10,000 Monte-Carlo samples. All curves plot the mean and shade the range over three global random seeds.}
    \label{fig:expt}
\end{figure}

%% file: section/8-discuss.tex
\section{Conclusion}

This paper has explored means of computing the information content $\mathcal{I}[V_S]$ of trained neural networks. First, an approach of \citet{denker} was considered, based on counting symmetries. Unfortunately, the resulting bound on information content seemed to reduce to simple parameter counting at large width. Next, an approach of \citet{valle-perez2018deep} based on the neural network--Gaussian process correspondence was considered. The technique was refined based on a characterisation of Gaussian orthant probabilities (Lemma \ref{thm:orthant}). The resulting generalisation bound was tested empirically, and found to be both non-vacuous and correlated with empirical generalisation performance at finite width. Overall, this paper's results support the findings of \citet{valle-perez2018deep}: a careful PAC-Bayesian analysis of infinite width networks offers a compelling explanation for the generalisability of neural networks with vastly more parameters than data.

Finally, Theorem \ref{thm:gen} may be converted to a generalisation bound for \textit{Gaussian process classification}, meaning it may be of independent interest.

%% file: section/99-app.tex
\section{Proofs}\label{app:proofs}

The following PAC-Bayesian theorem is due to \citet{valle-perez2018deep}:
\pacbayes*
\begin{proof}
The first inequality is a basic property of logarithms. The second inequality follows from \citet{Langford01boundsfor}'s Theorem 3, by setting the prior measure to $P(\cdot)=\Probe(\cdot)$ and the posterior measure to the conditional $Q(\cdot)=\Probe(\cdot|V_S)$. Under these settings, the average training error rate over the posterior is zero and $\mathrm{KL}(Q|P)=\log \tfrac{1}{\Probe[V_S]}=:\mathcal{I}[V_S]$.
\end{proof}

The essence of the following lemma is due to \citet{radford}. The lemma will be used in the proof of Theorem \ref{thm:relu}.

\input{theorem/nngp}

While condition (i) may seem non-trivial, notice that the lemma propagates this condition to the next layer via entailment (1). This means that provided condition (i) holds for $\phi(z^{(1)}(x))$ at the first layer, then recursive application of the lemma implies that the network's pre-activations are jointly Normal \textit{at all layers} via entailment (2).

\begin{proof}[Proof of Lemma \ref{lem:nngp}] 
To establish entailment (1), consider the $d_l$-dimensional vector $Z_1 := \left[z^{(l)}_1(x), ..., z^{(l)}_{d_l}(x)\right]$. Observe that $Z_1$ satisfies:
\begin{gather}\label{eq:z1}
    Z_1 = \frac{1}{\sqrt{d_{l-1}}}\sum_{j=1}^{d_{l-1}} \left[W^{(l)}_{1j}\phi \left(z^{(l-1)}_j(x)\right), ..., W^{(l)}_{d_{l}j}\phi \left(z^{(l-1)}_j(x)\right) \right].
\end{gather}
By conditions (i) and (ii), the summands in Equation \ref{eq:z1} are iid random vectors with zero mean, and any two distinct components of the same vector summand have the same variance and zero covariance. Then by the multivariate central limit theorem \citep[p.~16]{vaart_1998}, in the limit that $d_{l-1}\to\infty$, the components of $Z_1$ are Gaussian with a covariance equal to a scaled identity matrix. In particular, the components of $Z_1$ are iid with finite first and second moment. Applying condition (iii) then implies that the same holds for $\phi(Z_1)$. This establishes entailment (1).

To establish entailment (2), consider instead the $k$-dimensional vector $Z_2 := \left[z^{(l)}_i(x_1), ..., z^{(l)}_i(x_k)\right]$. Observe that $Z_2$ satisfies:
\begin{gather}\label{eq:z2}
    Z_2 = \frac{1}{\sqrt{d_{l-1}}}\sum_{j=1}^{d_{l-1}} \left[W^{(l)}_{ij}\phi \left(z^{(l-1)}_j(x_1)\right), ..., W^{(l)}_{ij}\phi \left(z^{(l-1)}_j(x_k)\right) \right].
\end{gather}
Again by combining conditions (i) and (ii), the summands in Equation \ref{eq:z2} are iid random vectors with finite mean and finite covariance. Then as $d_{l-1}\rightarrow\infty$, the distribution of $Z_2$ is jointly Normal---again by the multivariate central limit theorem. This establishes entailment (2).
\end{proof}

The essence of the following theorem appears in a paper by \citet{lee2018deep}, building on the work of \citet{choandsaul}. The theorem and its proof are included for completeness.
\relu*
\begin{proof}
    Condition (ii) of Lemma \ref{lem:nngp} holds at all layers for iid standard Normal weights, and condition (iii) holds trivially for the scaled relu nonlinearity. Provided we can establish condition (i) for the first layer activations $\phi\left(z^{(1)}_1(x)\right),...,\phi\left(z^{(1)}_{d_{l-1}}(x)\right)$, then condition (i) will hold at all layers by recursive application of Lemma \ref{lem:nngp}, thus establishing joint Normality of the pre-activations at all layers (including the network outputs). But condition (i) holds at the first layer, since it is quick to check by Equation \ref{eq:first-layer} that for any $x$ satisfying $\norm{x}_2 = \sqrt{d_0}$, the pre-activations $z^{(1)}_1(x),...,z^{(1)}_{d_{l-1}}(x)$ are iid $\mathcal{N}\left(0,1\right)$, and $\phi$ preserves both iid-ness and finite-ness of the first and second moment.
    
    Since the pre-activations at any layer are jointly Normal, all that remains is to compute their first and second moments. For the $i$th hidden unit in the $l$th layer, the first moment $\Expect \left[ z^{(l)}_i(x)\right] = 0$. This can be seen by taking the expectation of Equation \ref{eq:layer} and using the fact that the $W_{ij}^{(l)}$ are independent of the previous layer's activations and have mean zero.
    
    Since the pre-activations $z_i^{(l)}(x)$ and $z_i^{(l)}(x^\prime)$ are jointly Normal with mean zero, their distribution is completely described by their covariance matrix $\Sigma_l(x,x^\prime)$, defined by:
    \begin{align*}
        \rho_{l}(x,x^\prime)&:=\Expect \left[ z^{(l)}_i(x) z^{(l)}_i(x^\prime)\right] \\
        \Sigma_{l}(x,x^\prime)&:=
          \begin{bmatrix}
            \rho_{l}(x,x) & \rho_{l}(x,x^\prime)  \\
            \rho_{l}(x,x^\prime) & \rho_{l}(x^\prime,x^\prime)
          \end{bmatrix},
    \end{align*}
    where the hidden unit index $i$ is unimportant since hidden units in the same layer have identical distributions.
    
    The theorem statement will follow from an effort to express $\Sigma_l(x,x^\prime)$ in terms of $\Sigma_{l-1}(x,x^\prime)$, and then recursing back through the network. By Equation \ref{eq:layer} and independence of the $W_{ij}^{(l)}$, the covariance $\rho_l(x,x^\prime)$ may be expressed as: 
    \begin{equation}\label{eq:covar}
        \rho_{l}(x,x^\prime) = \Expect \left[ \phi \left(z^{(l-1)}_j(x)\right) \phi \left(z^{(l-1)}_j(x^\prime)\right) \right],
    \end{equation}
    where $j$ indexes an arbitrary hidden unit in the $(l-1)$th layer. To make progress, it helps to first evaluate:
    \begin{gather*}
        \rho_l(x,x) = \Expect \left[ \phi \left(z^{(l-1)}_j(x)\right)^2 \right] = \half \cdot 2 \cdot \rho_{l-1}(x,x),
    \end{gather*}
    which follows by the definition of $\phi$ and symmetry of the Gaussian expectation around zero. Then, by recursion:
    \begin{gather*}
        \rho_{l}(x,x) = \rho_{l-1}(x,x) =... = \rho_{1}(x,x) = 1,
    \end{gather*}
    where the final equality holds because the first layer pre-activations are iid $\mathcal{N}(0,1)$ by Equation \ref{eq:first-layer}. Therefore, the covariance $\Sigma_{l-1}$ at layer $l-1$ is just:
    \begin{equation*}
        \Sigma_{l-1}(x,x^\prime)=
          \begin{bmatrix}
            1 & \rho_{l-1}(x,x^\prime)  \\
            \rho_{l-1}(x,x^\prime) & 1
          \end{bmatrix},
    \end{equation*}
    Equation \ref{eq:covar} may now be used to express $\rho_l(x,x^\prime)$ in terms of $\rho_{l-1}(x,x^\prime)$. Dropping the $(x,x^\prime)$ indexing for brevity:
    \begin{align*}
        \rho_l &= \Expect_{u,v\sim \mathcal{N}\left(0,\Sigma_{l-1}\right)} \left[ \phi \left(u\right) \phi \left(v\right) \right] \\
        &= \frac{1}{\pi \sqrt{1-\rho_{l-1}^2}} \iint_{u,v\geq0}\idiff{u}\idiff{v} \exp\left[-\frac{u^2 - 2 \rho_{l-1} uv + v^2}{2(1-\rho_{l-1}^2)}\right]uv.
    \end{align*}
    By making the substitution $\rho_{l-1}=\cos\theta$, this integral becomes equivalent to $\frac{1}{\pi}J_1(\theta)$ as expressed in \citet[Equation 15]{choandsaul}. Substituting in the evaluation of this integral \citep[Equation 6]{choandsaul}, we obtain:
    \begin{equation}\label{eq:recur}
      \rho_{l}(x,x^\prime) = h(\rho_{l-1}(x,x^\prime)),
    \end{equation}
    where $h(t):= \tfrac{1}{\pi}\left[\sqrt{1-t^2} + t\cdot(\pi- \arccos t)\right]$. 
    
    All that remains is to evaluate $\rho_1(x,x^\prime)$. Since $\Expect\left[W^{(1)}_{ij}W^{(1)}_{ik}\right]=\delta_{jk}$, this is given by:
    \begin{align*}
        \rho_1(x,x^\prime) &:= \Expect \left[ z^{(1)}_i(x) z^{(1)}_i(x^\prime)\right] \\
        &= \frac{1}{d_0} \sum_{j,k=1}^{d_0} \Expect\left[W^{(1)}_{ij}W^{(1)}_{ik}\right] x_j x_k^\prime = \frac{x^Tx^\prime}{d_0}.
    \end{align*}
    The proof is completed by combining this expression for $\rho_1(x,x^\prime)$ with the recurrence relation in Equation \ref{eq:recur}.
\end{proof}

The following lemma is a main technical contribution of this paper.

\orthant*
\begin{proof} The orthant probability may first be expressed using the probability density function of the multivariate Normal distribution as follows:
    \begin{align*}
        p &= \frac{1}{\sqrt{(2\pi)^n\det\Sigma}}\int_{c\odot z\geq0}\econst^{-\half z^T \Sigma^{-1}z}\idiff{z}.
    \end{align*}
    By the change of variables $u = \frac{c\odot z}{\sqrt[2n]{\det\Sigma}}$ or equivalently $z = \sqrt[2n]{\det\Sigma}(c\odot u)$, the orthant probability may be expressed as:
    \begin{align*}
        p &= \frac{1}{\sqrt{(2\pi)^n}}\int_{u\geq0}\econst^{-\half (c\odot u)^T \sqrt[n]{\det \Sigma}\Sigma^{-1}(c\odot u)}\idiff{u} \\
        &= \frac{1}{2^n}\frac{1}{\sqrt{(2\pi)^n}}\int_{\R^n}\econst^{-\half (c\odot \abs{u})^T \sqrt[n]{\det \Sigma}\Sigma^{-1}(c\odot \abs{u})}\idiff{u}\\
        &= \frac{1}{2^n}\Expect_{u\sim\mathcal{N}(0,\mathbb{I})}\left[\econst^{-\half (c\odot \abs{u})^T\left( \sqrt[n]{\det \Sigma}\Sigma^{-1}-\mathbb{I}\right)(c\odot \abs{u})}\right],
    \end{align*}
    where the second equality follows by symmetry, and the third equality follows by inserting a factor of $\econst^{-u^2/2} \econst^{+u^2/2}=1$ into the integrand.
    
    Next, by Jensen's inequality,
    \begin{align*}
    p &\geq \frac{1}{2^n}\econst^{-\half \Expect_{u\sim\mathcal{N}(0,\mathbb{I})}\left[(c\odot \abs{u})^T\left( \sqrt[n]{\det \Sigma}\Sigma^{-1}-\mathbb{I}\right)(c\odot \abs{u})\right]} \\
    &= \frac{1}{2^n}\econst^{-\half \sum_{ij} \Expect_{u\sim\mathcal{N}(0,\mathbb{I})}\left[c_ic_j\abs{u_i}\abs{u_j}\left( \sqrt[n]{\det \Sigma}\Sigma^{-1}_{ij}-\delta_{ij}\right)\right]} \\
    &= \frac{1}{2^n}\econst^{-\half \left[ \sum_i \left( \sqrt[n]{\det \Sigma}\Sigma^{-1}_{ii}-1\right) + \frac{2}{\pi}\sum_{i\neq j} c_ic_j\sqrt[n]{\det \Sigma}\Sigma^{-1}_{ij}\right]} \\
    &= \frac{1}{2^n}\econst^{-\half \left[ \sum_i \left( (1-\frac{2}{\pi})\sqrt[n]{\det \Sigma}\Sigma^{-1}_{ii}-1\right) + \frac{2}{\pi}\sum_{ij} c_ic_j\sqrt[n]{\det \Sigma}\Sigma^{-1}_{ij}\right]}\\
    &= \frac{1}{2^n}\econst^{-\half \left[ (1-\frac{2}{\pi})\sqrt[n]{\det \Sigma}\trace(\Sigma^{-1})-n + \frac{2}{\pi} \sqrt[n]{\det \Sigma}c^T\Sigma^{-1}c\right]} \\
    &= \frac{1}{2^n}\econst^{\frac{n}{2}-\sqrt[n]{\det \Sigma}\left[ (\frac{1}{2}-\frac{1}{\pi})\trace(\Sigma^{-1}) + \frac{1}{\pi} c^T\Sigma^{-1}c\right]}.
    \end{align*}
    The third equality follows by noting $\Expect_{u\sim\mathcal{N}(0,\mathbb{I})}\abs{u_i}\abs{u_i}=\Expect_{u\sim\mathcal{N}(0,1)}u^2=1$, while for $i\neq j$, $\Expect_{u\sim\mathcal{N}(0,\mathbb{I})}\abs{u_i}\abs{u_j}=[\Expect_{u\sim\mathcal{N}(0,1)}\abs{u}]^2=2/\pi$.
\end{proof}

%% file: theorem/nngp.tex
\begin{restatable}[NNGP correspondence]{lemma}{nngp}
\label{lem:nngp} For the neural network layer given by Equation \ref{eq:layer}, consider randomly sampling the weight matrix $W^{(l)}$. If the following hold:
\begin{enumerate}[label=(\roman*)]
    \item for every $x \in \R^{d_0}$, the activations $\phi\left(z^{(l-1)}_1(x)\right),...,\phi\left(z^{(l-1)}_{d_{l-1}}(x)\right)$ are iid with finite first and second moment;
    \item the weights $W_{ij}^{(l)}$ are drawn iid with zero mean and finite variance;
    \item for any random variable $z$ with finite first and second moment, $\phi(z)$ also has finite first and second moment;
\end{enumerate}
then, in the limit that $d_{l-1}\rightarrow\infty$, the following also hold:
\begin{enumerate}[label=(\arabic*)]
    \item for every $x \in \R^{d_0}$, the activations $\phi\left(z^{(l)}_1(x)\right),...,\phi\left(z^{(l)}_{d_{l}}(x)\right)$ are iid with finite first and second moment;
    \item for any collection of $k$ inputs $x_1, ..., x_k$, the distribution of the $i$th pre-activations $z^{(l)}_i(x_1), ..., z^{(l)}_i(x_k)$ is jointly Normal.
\end{enumerate}
\end{restatable}

%% file: main.bbl
\begin{thebibliography}{34}
\providecommand{\natexlab}[1]{#1}
\providecommand{\url}[1]{\texttt{#1}}
\expandafter\ifx\csname urlstyle\endcsname\relax
  \providecommand{\doi}[1]{doi: #1}\else
  \providecommand{\doi}{doi: \begingroup \urlstyle{rm}\Url}\fi

\bibitem[Akaike(1974)]{akaike}
Hirotugu Akaike.
\newblock A new look at the statistical model identification.
\newblock \emph{Automatic Control}, 1974.

\bibitem[Arpit et~al.(2017)Arpit, Jastrz{\k{e}}bski, Ballas, Krueger, Bengio,
  Kanwal, Maharaj, Fischer, Courville, Bengio, and
  Lacoste-Julien]{pmlr-v70-arpit17a}
Devansh Arpit, Stanis{\l}aw Jastrz{\k{e}}bski, Nicolas Ballas, David Krueger,
  Emmanuel Bengio, Maxinder~S. Kanwal, Tegan Maharaj, Asja Fischer, Aaron
  Courville, Yoshua Bengio, and Simon Lacoste-Julien.
\newblock A closer look at memorization in deep networks.
\newblock In \emph{International Conference on Machine Learning}, 2017.

\bibitem[{Bartol Jr.} et~al.(2015){Bartol Jr.}, Bromer, Kinney, Chirillo,
  Bourne, Harris, and Sejnowski]{nanoconnectomic}
Thomas~M. {Bartol Jr.}, Cailey Bromer, Justin Kinney, Michael~A. Chirillo,
  Jennifer~N. Bourne, Kristen~M. Harris, and Terrence~J. Sejnowski.
\newblock Nanoconnectomic upper bound on the variability of synaptic
  plasticity.
\newblock \emph{eLife}, 2015.

\bibitem[Blumer et~al.(1987)Blumer, Ehrenfeucht, Haussler, and
  Warmuth]{BLUMER1987377}
Anselm Blumer, Andrzej Ehrenfeucht, David Haussler, and Manfred~K. Warmuth.
\newblock Occam's razor.
\newblock \emph{Information Processing Letters}, 1987.

\bibitem[Chizat et~al.(2019)Chizat, Oyallon, and Bach]{NEURIPS2019_ae614c55}
L\'{e}na\"{\i}c Chizat, Edouard Oyallon, and Francis Bach.
\newblock On lazy training in differentiable programming.
\newblock In \emph{Neural Information Processing Systems}, 2019.

\bibitem[Cho and Saul(2009)]{choandsaul}
Youngmin Cho and Lawrence Saul.
\newblock Kernel methods for deep learning.
\newblock In \emph{Neural Information Processing Systems}, 2009.

\bibitem[De~Palma et~al.(2019)De~Palma, Kiani, and Lloyd]{NEURIPS2019_feab05aa}
Giacomo De~Palma, Bobak Kiani, and Seth Lloyd.
\newblock Random deep neural networks are biased towards simple functions.
\newblock In \emph{Neural Information Processing Systems}, 2019.

\bibitem[Denker and Wittner(1988)]{denker}
John Denker and Ben Wittner.
\newblock Network generality, training required, and precision required.
\newblock In \emph{Neural Information Processing Systems}, 1988.

\bibitem[Dziugaite and Roy(2017)]{DR17}
Gintare~Karolina Dziugaite and Daniel~M. Roy.
\newblock Computing nonvacuous generalization bounds for deep (stochastic)
  neural networks with many more parameters than training data.
\newblock In \emph{Uncertainty in Artificial Intelligence}, 2017.

\bibitem[Genz and Bretz(2009)]{Genz2009ComputationOM}
Alan Genz and Frank Bretz.
\newblock \emph{Computation of Multivariate Normal and t Probabilities}.
\newblock Springer-Verlag Berlin Heidelberg, 2009.

\bibitem[Glorot and Bengio(2010)]{pmlr-v9-glorot10a}
Xavier Glorot and Yoshua Bengio.
\newblock Understanding the difficulty of training deep feedforward neural
  networks.
\newblock In \emph{Artificial Intelligence and Statistics}, 2010.

\bibitem[Hinton and van Camp(1993)]{Hinton1993KeepingTN}
Geoffrey~E. Hinton and Drew van Camp.
\newblock Keeping neural networks simple by minimizing the description length
  of the weights.
\newblock In \emph{Conference on Learning Theory}, 1993.

\bibitem[Jacot et~al.(2018)Jacot, Gabriel, and Hongler]{NEURIPS2018_5a4be1fa}
Arthur Jacot, Franck Gabriel, and Clement Hongler.
\newblock Neural tangent kernel: Convergence and generalization in neural
  networks.
\newblock In \emph{Neural Information Processing Systems}, 2018.

\bibitem[Koulakov et~al.(2021)Koulakov, Shuvaev, and Zador]{genomic}
Alexei Koulakov, Sergey Shuvaev, and Anthony Zador.
\newblock Encoding innate ability through a genomic bottleneck, 2021.
\newblock Personal communication.

\bibitem[Landauer(1986)]{landauer}
Thomas~K. Landauer.
\newblock How much do people remember? {S}ome estimates of the quantity of
  learned information in long-term memory.
\newblock \emph{Cognitive Science}, 1986.

\bibitem[Langford and Seeger(2001)]{Langford01boundsfor}
John Langford and Matthias Seeger.
\newblock Bounds for averaging classifiers.
\newblock Technical report, Carnegie Mellon University, 2001.

\bibitem[LeCun et~al.(1998)LeCun, Cortes, and Burges]{lecun2010mnist}
Yann LeCun, Corinna Cortes, and Christopher~J.C. Burges.
\newblock {MNIST} handwritten digit database, 1998.

\bibitem[Lee et~al.(2018)Lee, Sohl-Dickstein, Pennington, Novak, Schoenholz,
  and Bahri]{lee2018deep}
Jaehoon Lee, Jascha Sohl-Dickstein, Jeffrey Pennington, Roman Novak, Sam
  Schoenholz, and Yasaman Bahri.
\newblock Deep neural networks as {G}aussian processes.
\newblock In \emph{International Conference on Learning Representations}, 2018.

\bibitem[Lee et~al.(2019)Lee, Xiao, Schoenholz, Bahri, Novak, Sohl-Dickstein,
  and Pennington]{NEURIPS2019_0d1a9651}
Jaehoon Lee, Lechao Xiao, Samuel Schoenholz, Yasaman Bahri, Roman Novak, Jascha
  Sohl-Dickstein, and Jeffrey Pennington.
\newblock Wide neural networks of any depth evolve as linear models under
  gradient descent.
\newblock In \emph{Neural Information Processing Systems}, 2019.

\bibitem[Linial et~al.(1991)Linial, Mansour, and Rivest]{linial}
Nathan Linial, Yishay Mansour, and Ronald~L. Rivest.
\newblock Results on learnability and the {V}apnik-{C}hervonenkis dimension.
\newblock \emph{Information and Computation}, 1991.

\bibitem[Liu et~al.(2021)Liu, Bernstein, Meister, and Yue]{nero2021}
Yang Liu, Jeremy Bernstein, Markus Meister, and Yisong Yue.
\newblock Learning by turning: Neural architecture aware optimisation.
\newblock \emph{arXiv:2102.07227}, 2021.

\bibitem[McAllester(1998)]{McAllester}
David~A. McAllester.
\newblock Some {PAC}-{B}ayesian theorems.
\newblock In \emph{Conference on Computational Learning Theory}, 1998.

\bibitem[Mingard et~al.(2020)Mingard, Valle-P{\'e}rez, Skalse, and
  Louis]{Mingard2020IsSA}
Chris Mingard, Guillermo Valle-P{\'e}rez, Joar Skalse, and Ard~A. Louis.
\newblock Is {SGD} a {B}ayesian sampler? {W}ell, almost.
\newblock \emph{arXiv:2006.15191}, 2020.

\bibitem[Mollica and Piantadosi(2019)]{Mollica2019HumansSA}
Frank Mollica and Steven Piantadosi.
\newblock Humans store about 1.5 megabytes of information during language
  acquisition.
\newblock \emph{Royal Society Open Science}, 2019.

\bibitem[Neal(1994)]{radford}
Radford~M. Neal.
\newblock \emph{Bayesian Learning for Neural Networks}.
\newblock {Ph.D.} thesis, Department of Computer Science, University of
  Toronto, 1994.

\bibitem[Rissanen(1986)]{rissanen1986}
Jorma Rissanen.
\newblock Stochastic complexity and modeling.
\newblock \emph{Annals of Statistics}, 1986.

\bibitem[Schmidhuber(1997)]{SCHMIDHUBER1997857}
Jürgen Schmidhuber.
\newblock Discovering neural nets with low {K}olmogorov complexity and high
  generalization capability.
\newblock \emph{Neural Networks}, 1997.

\bibitem[Schwarz(1978)]{bic}
Gideon Schwarz.
\newblock Estimating the dimension of a model.
\newblock \emph{The Annals of Statistics}, 1978.

\bibitem[Seeger(2003)]{seeger}
Matthias Seeger.
\newblock {PAC}-{B}ayesian generalisation error bounds for {G}aussian process
  classification.
\newblock \emph{Journal of Machine Learning Research}, 2003.

\bibitem[Shawe-Taylor and Williamson(1997)]{ShaweTaylor1997APA}
John Shawe-Taylor and Robert~C. Williamson.
\newblock A {PAC} analysis of a {B}ayesian estimator.
\newblock In \emph{Conference on Learning Theory}, 1997.

\bibitem[Shawe-Taylor et~al.(1996)Shawe-Taylor, Bartlett, Williamson, and
  Anthony]{ShaweTaylor1996AFF}
John Shawe-Taylor, Peter Bartlett, Robert~C. Williamson, and Martin Anthony.
\newblock A framework for structural risk minimisation.
\newblock In \emph{Conference on Learning Theory}, 1996.

\bibitem[Ulyanov et~al.(2018)Ulyanov, Vedaldi, and
  Lempitsky]{Ulyanov2018DeepIP}
Dmitry Ulyanov, Andrea Vedaldi, and Victor Lempitsky.
\newblock Deep image prior.
\newblock \emph{Computer Vision and Pattern Recognition}, 2018.

\bibitem[Valle-P{\'e}rez et~al.(2019)Valle-P{\'e}rez, Camargo, and
  Louis]{valle-perez2018deep}
Guillermo Valle-P{\'e}rez, Chico~Q. Camargo, and Ard~A. Louis.
\newblock Deep learning generalizes because the parameter-function map is
  biased towards simple functions.
\newblock In \emph{International Conference on Learning Representations}, 2019.

\bibitem[van~der Vaart(1998)]{vaart_1998}
Aad~W. van~der Vaart.
\newblock \emph{Asymptotic Statistics}.
\newblock Cambridge University Press, 1998.

\end{thebibliography}
